\documentclass{article}

     \usepackage[preprint]{neurips_2020}

\usepackage[utf8]{inputenc} 
\usepackage[T1]{fontenc}    
\usepackage{hyperref}       
\usepackage{url}            
\usepackage{booktabs}       
\usepackage{amsfonts}       
\usepackage{nicefrac}       
\usepackage{microtype}      

\usepackage{times}
\usepackage{epsfig}
\usepackage{graphicx}
\usepackage{amsmath}
\usepackage{bbm, dsfont}
\usepackage{amsthm}
\usepackage{amssymb}
\usepackage{booktabs}
\usepackage[dvipsnames, table]{xcolor}
\definecolor{Gray}{gray}{0.85}
\usepackage{multirow}
\usepackage{pifont}
\usepackage{algorithm}
\usepackage{algpseudocode}
\usepackage{comment}
\usepackage{hyperref}
\usepackage{subcaption}
\usepackage{caption}
\usepackage{tikz}
\usepackage{comment}
\usepackage{amsmath,amssymb} 
\usepackage{color}

\newtheorem{theorem}{Theorem}

\title{To Fold or Not to Fold: a Necessary and Sufficient Condition on Batch-Normalization Layers Folding}

\author{%
  Edouard Yvinec$^{1,2}$ \And Arnaud Dapogny$^2$ \And Matthieu Cord$^{2,3}$ \And Kevin Bailly$^{1,2}$  \\
  $^1$ Sorbonne Université, CNRS, ISIR, f-75005, 4 Place Jussieu 75005 Paris, France\\
  $^2$ Datakalab, 114 boulevard Malesherbes, 75017 Paris, France \\
  $^3$ Valeo, 100 rue de Courcelles, 75017 Paris, France \\
  \texttt{ey@datakalab.com}\\
}

\begin{document}

\maketitle
\begin{abstract}
    Batch-Normalization (BN) layers have become fundamental components in the evermore complex deep neural network architectures. Such models require acceleration processes for deployment on edge devices. However, BN layers add computation bottlenecks due to the sequential operation processing: thus, a key, yet often overlooked component of the acceleration process is BN layers folding. In this paper, we demonstrate that the current BN folding approaches are suboptimal in terms of how many layers can be removed. We therefore provide a necessary and sufficient condition for BN folding and a corresponding optimal algorithm. The proposed approach systematically outperforms existing baselines and allows to dramatically reduce the inference time of deep neural networks.

\end{abstract}
\section{Introduction}
Deep Neural Networks (DNNs) are ubiquitous in various sub-domains of computer vision, e.g. in Image Classification \cite{he2016deep}, Object Detection \cite{he2017mask} or Semantic Segmentation \cite{chen2017deeplab}.
Such models evolved from narrow architectures \cite{lecun1989backpropagation}, to deeper and more complex architectures \cite{tan2019efficientnet}.
The main concern in the early developments of very deep neural networks was the efficiency of the learning process \cite{krizhevsky2012imagenet} with batch-normalization (BN) layers \cite{ioffe2015batch} being a prime example of solution to this challenge. BN layers center and reduce the features which solves the internal covariate shift problem arising from the change of distribution of the layer inputs.
It was so effective that it became almost mandatory in modern architectures.

However, very deep neural network deployment is still limited due to their heavy computational requirements. To tackle this problem, many solutions for DNN acceleration have been developed including, but not limited to, pruning \cite{frankle2018lottery,lin2020hrank,yvinec2021red} and quantization \cite{zhao2019improving,meller2019same,finkelstein2019fighting}. Pruning consists in removing a fraction of the mathematical operations performed within each layers while quantization consists in reducing the temporal cost of scalar operations by using lower bit-wise representations. However most of these methods leverage the dependencies across consecutive convolutional and fully-connected layers. Those dependencies are affected by BN layers. This is the first motivation for BN folding. 
It consists in removing the BN layers from the graph by updating the remaining parameters to keep the predictive function unchanged. It facilitates the use of most pruning and quantization protocols.

The second motivation is straightforward: the lower the number of layers, the faster the inference, all other factors being equal, which is the case with BN folding \cite{jacob2018quantization}. Although BN operations adds very few parameters to the network, the induced computational cost is important because of the sequentiality of the operations. The folding process removes this bottleneck and without changing the predictive function but it is still often overlooked in the literature.
\begin{figure}[!t]
    \centering
    \includegraphics[width = \linewidth]{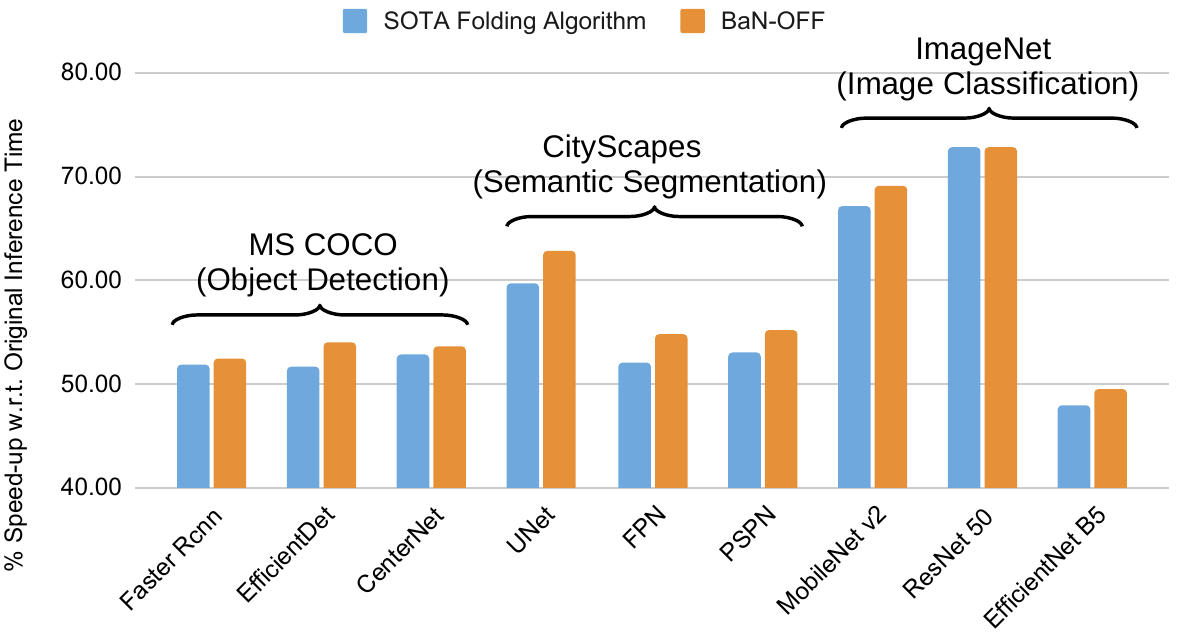}
     \caption{Comparison between the state-of-the-art BN folding algorithm and the proposed BaN-OFF method on several architectures and tasks. Both methods significantly improve the inference speed at virtually no cost. The proposed method systematically matches or outperforms the already existing method on all configurations.}
    \label{fig:main}
\end{figure}

For those two reasons, BN folding is usually presented in a naive way where BN layers are folded if and only if they are immediately connected to a convolutional layer or a fully-connected layer. Although, this condition is sufficient for a BN layer to be foldable, it is far from necessary meaning that it can't be optimal. In this study, we determine a necessary and sufficient condition for a BN layer to be foldable which translates into an algorithm for optimal batch-normalization layers folding, dubbed \textbf{BaN-OFF}.
Our tests are two-folds: first, we validate the systematic improvement over the standard folding protocol. Second, we demonstrate the importance as a compression step in more complex pipeline.
Our contributions are:
\begin{itemize}
    \item From a theoretical standpoint: a novel necessary and sufficient condition on the feasibility of batch normalization layers folding in any neural network architecture.
    \item From a practical standpoint: an algorithm which implements the proposed optimal condition and outperforms the existing folding methods at virtually no cost in pre-processing time.
    \item Outstanding inference acceleration results on several tasks and architectures and systematic improvement from the standard folding algorithm, see Figure \ref{fig:main}.
\end{itemize}
In this work, we focused on BN as defined by \cite{ioffe2015batch}, however, all our results extend to any normalization layers that implements static affine transformations such as Group Normalization \cite{wu2018group} or Switchable Normalization \cite{luo2019switchable}.

\section{Related Work}
Batch-normalization, introduced by Ioffe \textit{et al.} in \cite{ioffe2015batch}, had and continues to have a tremendous impact in deep learning. Prior to it, architectures such as VGG \cite{simonyan2014very} achieved remarkable results on ImageNet \cite{ImageNet_cvpr09} but were difficult to train. However as discussed in \cite{martin2019traditional} where the authors drew a comparison between the standard VGG and VGG-BN where additional BN layers are used. It appeared that such layers drastically improve the performance while simplifying the training process. Concurrently, these layers have been used in almost all novel architectures (\cite{ren2015faster,he2016deep,szegedy2016rethinking,lin2017feature,zhao2017pyramid,sandler2018mobilenetv2,zhou2019objects,tan2019efficientnet,tan2020efficientdet}) to the point that they are now considered an unavoidable part of the standard convolutional and fully-connected layers.

\paragraph{Inference Acceleration}
There are two main approaches to DNN inference acceleration: pruning and quantization.
Pruning consists in removing elements of the graph defined by the DNN \cite{renda2020comparing}. Quantization consists in mapping the weight values from continuous values to evenly-spread integer values using lower bit-wise representations (e.g. int8 or int4) \cite{krishnamoorthi2018quantizing}. Both approaches are often combined or even intrinsically leverage BN folding.

\paragraph{Batch-normalization Folding:} BN folding consists in removing such BN layers from the network's graph by updating the appropriate adjacent layers. This operation is performed if and only if the structural changes don't affect the predictive function associated to the network. The naive approach \cite{jacob2018quantization} is widely documented and applied in DNN acceleration. However it is not optimal as it doesn't fold every foldable BN layers. To solve this issue, we propose BaN-OFF, a necessary and sufficient condition for such layers to be folded as well as its corresponding algorithm.

\paragraph{Batch-normalization Folding and other Accelerations:}
 In \cite{yvinec2021red}, authors propose to prune similar neurons based on their weight representations. To preserve the predictive function they update the consecutive layers and to do so they perform BN folding. Other pruning mechanisms also leverage BN folding such as \cite{srinivas2015data,liu2017learning,ye2018rethinking}. \cite{nagel2019data} aim at balancing the weights, for quantization, across the entire network post BN folding and before performing quantization. Similarly, in \cite{zhao2019improving}, the method folds the BN layers before splitting outlier neurons in order to compactify the distribution of the weight values thus reducing the quantization error. 

Our novel algorithm allows to fold more BN layers than existing baselines, significantly reducing the inference time of deep neural networks. In addition, BAN-OFF allows to process more layers using existing pruning or quantization methods.

\section{Methodology}
Let $f$ be a feed forward neural network with $L$ layers ${(f_l)}_{l\in \{1,...,L\}}$. In this work, we define any layer as one of the following operations:
\begin{enumerate}
    \item expressive layers: a layer $f_l$ is expressive if and only if it can be expressed as the mathematical operation $f_l:x\mapsto W_lx +b_l$.
    \item non-affine layers: a layer $f_l$ is non-affine if and only if it corresponds to a mathematical operation that doesn't satisfy $f(ax+b) = \tilde a f(x) + \tilde b$.
    \item BN layers \cite{ioffe2015batch}: a layer $f_l$ of BN is defined, in the general case, by four parameters $\gamma_l$, $\beta_l$, $\mu_l$ and $\sigma_l$ such that $f:x\mapsto \gamma_l\frac{x-\mu_l}{\sigma_l + \epsilon}+\beta_l$ with $\epsilon$ a small constant positive scalar (usually $\epsilon=10^{-3}$).
    \item other layers.
\end{enumerate}
For instance, convolutional and fully-connected layers without their activation functions are the most common examples of expressive layers.
Non-affine layers correspond and are limited to the activation functions in most deep neural network architectures. Finally, pooling layers, concatenations, additions and flattening are examples of other layers commonly found in said architectures.

\subsection{Batch-Normalization Folding Limits}
\paragraph{Simplest Case:} let's assume $f$ is sequential with the following patterns: if $f_l$ is an expressive layer then $f_{l+1}$ is a BN layer and $f_{l+2}$ is a non-affine layer, \textit{i.e.} $f_{l+2}(f_{l+1}(f_{l}(x))) = \sigma\left(\text{BN}(W_lx+b_l)\right)$ where $\sigma$ is an activation function and $\text{BN}$ is a BN operation, see Figure \ref{fig:bn_folding} a). Under those assumptions all the BN layers can be folded using the naive BN-folding algorithm as follows: for each BN layer $l+1$ we remove the layer in the graph associated to the network and update the weights of $f_{l}$ with
\begin{equation}\label{eq:back_folding}
    \begin{cases}
    W_l \leftarrow \gamma_{l+1} \frac{W_l}{\sigma_{l+1} + \epsilon}\\
    b_l \leftarrow \gamma_{l+1} \frac{b_l - \mu_{l+1}}{\sigma_{l+1} + \epsilon} + \beta_{l+1}
    \end{cases}
\end{equation}
Then the resulting neural network is mathematically identical to the original network while requiring less operations to be computed.

\paragraph{Sequential Model:} Let's assume $f$ is sequential in the most general sense, \textit{i.e.} $f : x\mapsto f_L(\dots f_1(x))$, see Figure \ref{fig:bn_folding} b). The difference with the previous case is the possible presence of other layers between the BN layers and the expressive ones as well as the the possibility to have the activation before the BN. The naive BN-folding algorithm consists in folding the BN layers in the previous or next expressive layer as long as their are no intermediate non-affine layers. To fold a BN layer backward (\textit{i.e.} the expressive layer to update has a lower index) we use equation \ref{eq:back_folding}. On the other hand, to fold a BN layer forward (\textit{i.e.} the expressive layer to update has a higher index $l+k$) we use
\begin{equation}\label{eq:front_folding}
    \begin{cases}
    W_{l+k} \leftarrow \gamma_{l+1} \frac{W_{l+k}}{\sigma_{l+1} + \epsilon}\\
    b_{l+k} \leftarrow -\gamma_{l+1} \frac{\mu_{l+1}}{\sigma_{l+1} + \epsilon} + \beta_{l+1}W_{l+k} + b_{l+k}
    \end{cases}
\end{equation}
Under the sequential assumption, the algorithm folds all the BN layers that can be folded. Therefore this BN-folding algorithm is optimal in this specific case.

\paragraph{General DAG:} Let's assume the graph of $f$ is a directed acyclic graph (DAG), see Figure \ref{fig:bn_folding} c). Compared to a sequential model this corresponds to adding skip connections and possibly multiple leaves to the graph. In this case, the naive BN-folding algorithm consists in only considering sequential portions of the graph and limiting the resolution to the previous case. However a number of BN layers may be involved in non-sequential patterns as illustrated in Figure \ref{fig:bn_folding} c). Such BN layers can be folded by the proposed method.
As illustrated, this is done by folding the BN layer in the input expressive layers while updating adequately the the consecutive expressive layer (the one below the BN node in Figure \ref{fig:bn_folding} c).
For this reason the naive BN-folding algorithm is not optimal.
Now, we will define a necessary and sufficient condition for BN folding.

\begin{figure}[!t]
    \centering
    \includegraphics[width = 0.9\linewidth]{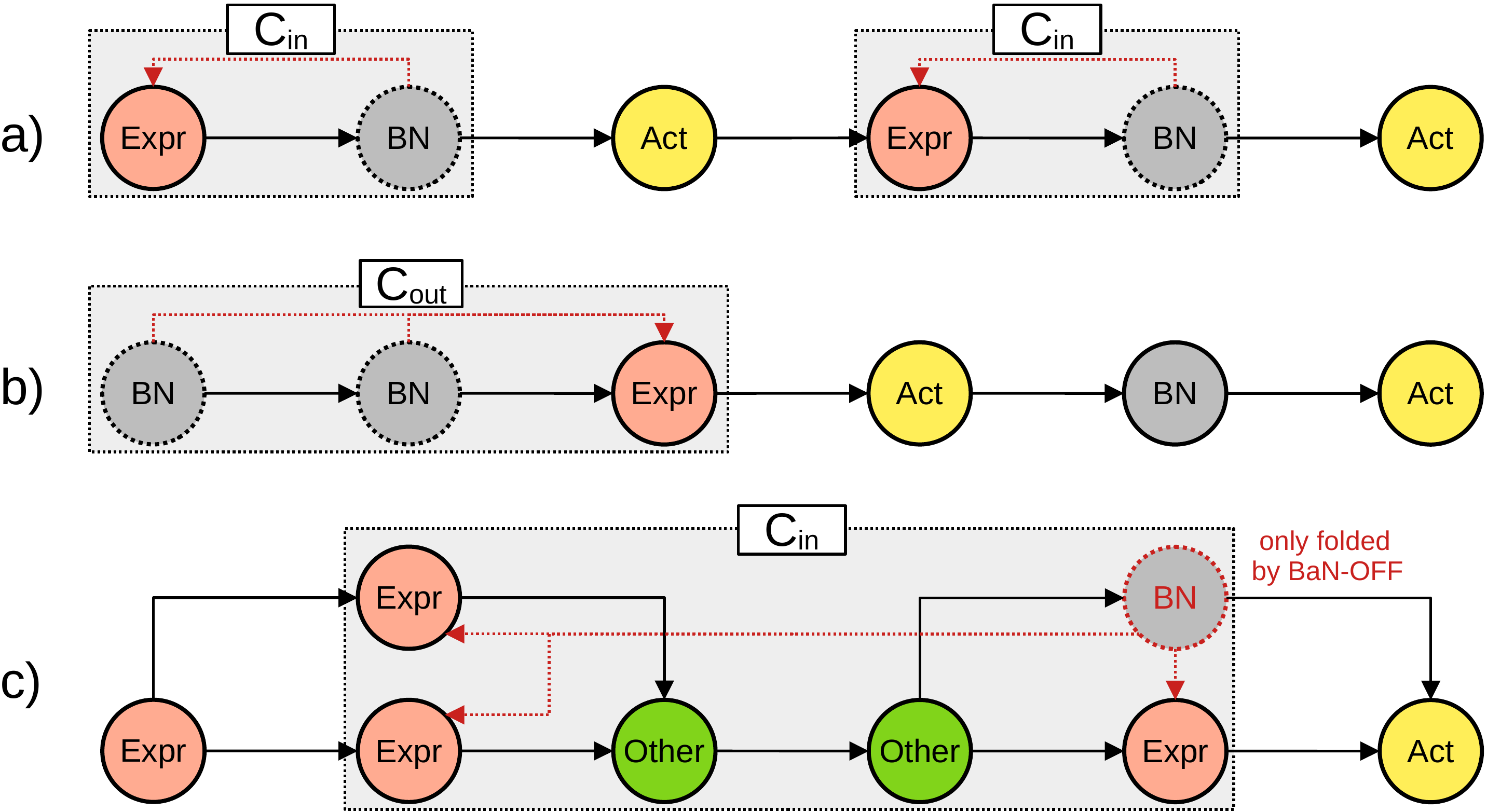}
     \caption{BN Folding of different neural networks archetypes. First a) a VGG with BN layers. Second an example of a general sequential model. Finally c) a network (DAG) with skip connections. Red nodes correspond to expressive layers, Yellow nodes to activation functions, grey nodes to BN layers and green nodes to other layers. The dark connections correspond to the graph associated to the network and dashed, red connections to the BN folding.}
    \label{fig:bn_folding}
    \label{fig:connected_comp}
\end{figure}
\subsection{Theoretical Folding Condition}
Every feed forward neural network architecture can be expressed as a DAG. Let $f$ be a DAG with $n$ nodes. We recall that, by definition, a layer $f_l$ is affine, if and only if $f_l(ax+b) = \tilde a f_l(x) + \tilde b$. The intuition is to isolate BN layers from non-affine layers.
To do so, we define the connected components $C$ of a node $N$ (corresponding to a BN layer) recursively. We start with the set $C=\{N\}$ and we add all the adjacent, affine nodes (\textit{i.e.} nodes not associated with non-affine layers). For each new, not-expressive, element of $C$ we repeat the process until there are no new elements.
The graph defined by $C$ is decomposed in the layers entering node $N$ and the layers leaving node $N$. We note the first part $C_{\text{in}}$ and the second $C_{\text{out}}$.
For a simple example (Figure \ref{fig:bn_folding} c), if we consider the VGG BN case and a BN layer $f_{l}$, we have $C_{\text{in}} = \{f_{l-1}, f_{l}\}$ with the input expressive layer $f_{l-1}$ and $C_{\text{out}} = \{f_{l}, f_{l+1}\}$ with $f_{l+1}$ the activation layer.
More examples of such connected components are illustrated in Figure \ref{fig:connected_comp}.
Then, for a given BN layer, the following theorem gives us our necessary and sufficient condition on the possibility to fold the layer. 
\begin{theorem}\label{thm:folding}
Let $f$ be a DAG associated to a neural network. A BN layer of node $N$ can be folded if and only if at least one of $C_{\text{in}}$ or $C_{\text{out}}$ satisfies:
\begin{enumerate}
    \item the set is not limited to $\{N\}$
    \item all the leaves (aside $N$) of the set are expressive layers 
\end{enumerate}
\end{theorem}
\begin{proof}
Here we will prove the necessity of the condition by assuming that either one of the condition is not satisfied by both $C_{\text{in}}$ and $C_{\text{out}}$. First, let's assume $C_{\text{in}}$ and $C_{\text{out}}$ are restricted to $\{N\}$, then it implies that the BN layer is "surrounded" by non-affine layers and thus can't be folded (second grey node of Figure \ref{fig:bn_folding} b). Second, let's assume that both $C_{\text{in}}$ and $C_{\text{out}}$ have at least one leaf that is not associated to an expressive layer aside from $N$ itself. Under this assumption, in $C$, there exists a path $f_{\rho(l+k)} \circ \dots \circ f_{\rho(l)}$ of length $k$ and a path $f_{\eta(l+k')} \circ \dots \circ f_{\eta(l)}$ of length $k'$ (see Figure \ref{fig:proof_illustration}) such that
\begin{enumerate}
    \item $f_{\eta(l+k')}$ corresponds to node $N$ and $f_{\eta(l+k')}\neq f_{\rho(l+k)}$
    \item $f_{\eta(l)}=f_{\rho(l)}$
\end{enumerate}
\begin{figure}[!t]
    \centering
    \includegraphics[width = 0.8\linewidth]{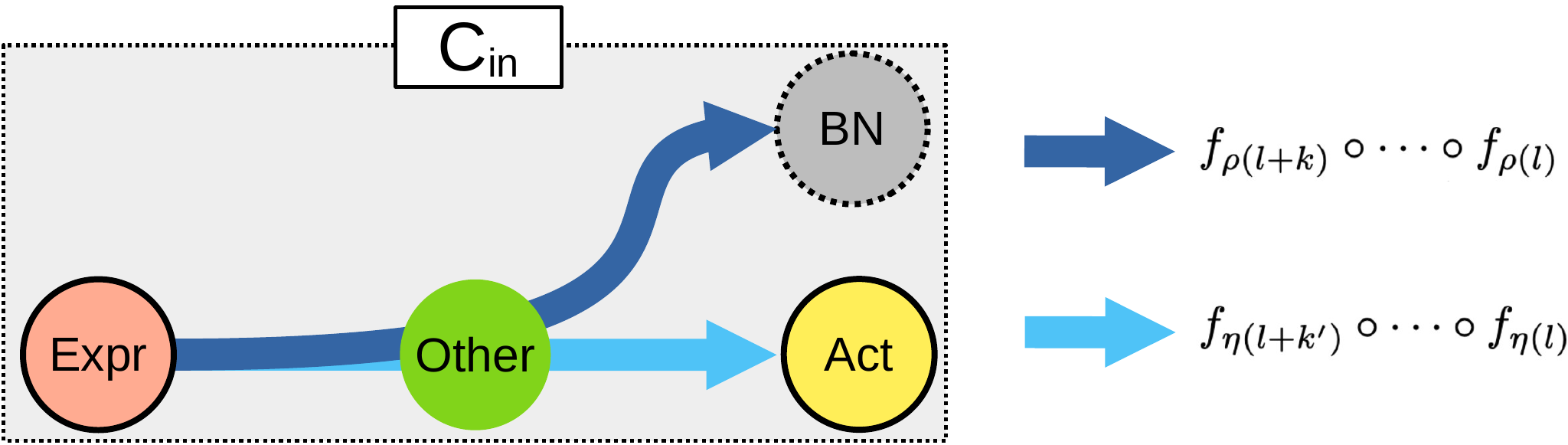}
     \caption{Two paths in the same $C_{\text{in}}$ sub-graph with different leaves $f_{\rho(l+k)}$ and $f_{\eta(l+k')}$. In such instance, the activation layer $f_{\rho(l+k)}$ receives the same processed features as the BN layer $f_{\eta(l+k')}$. Thus the folding can't be performed without altering the predictive function.}
    \label{fig:proof_illustration}
\end{figure}
By definition of $C$ all the consecutive layers of $f_{\rho(l+k)}$ are non-affine. Then if we fold $N$, the layer $f_{\eta(l)}$ won't compute the same operations and by definition, the path $f_{\rho(l)} \circ \dots \circ f_{\rho(l+k)}$ can't be updated to negate this effect. thus the node $N$ can't be folded. Such scenario is illustrated in Figure \ref{fig:proof} b).
\begin{figure}[!t]
    \centering
    \includegraphics[width = \linewidth]{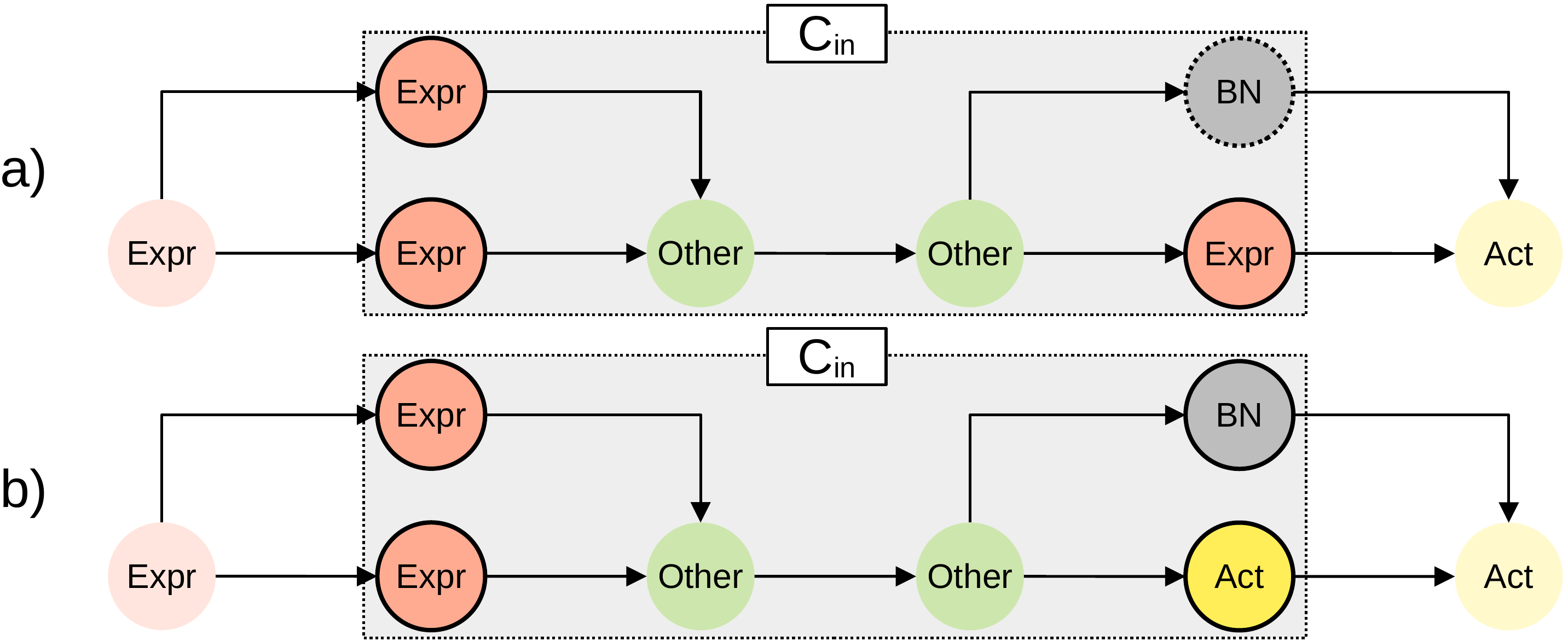}
     \caption{Critical examples of BN that can be folded a) and can't be folded b). The leaves of the sub-graphs $C_{\text{in}}$ are highlighted to show when the second condition of theorem \ref{thm:folding} is satisfied.}
    \label{fig:proof}
\end{figure}
This demonstrates the necessity of the proposed condition.
\end{proof}
Following this result, we can see in Figure \ref{fig:connected_comp} c) a BN layer that can be folded but not with the naive BN-folding algorithm.
To conclude our proof we will demonstrate that the condition is sufficient.

\subsection{Proposed Practical Algorithm}
Instead of an analytic proof we propose a constructive proof \textit{via} an algorithm which performs the BN folding as long as the condition from theorem \ref{thm:folding} is satisfied.
Let's assume $C_{\text{in}}$ satisfies the condition, then we separate the leaves in two categories: the ones with a path to node $N$, noted $I$ and the others $O$.
The nodes in $I$ are updated using equation \ref{eq:back_folding} and the others negate the previous update with the following update:
\begin{equation}\label{eq:update_back}
    \begin{cases}
    W \leftarrow \frac{W}{\gamma}(\sigma + \epsilon) \\
    b \leftarrow b + \frac{W\mu}{\sigma + \epsilon} \gamma - \beta W
    \end{cases}
\end{equation}
If $C_{\text{in}}$ doesn't satisfy the second condition from theorem \ref{thm:folding}, then we use equation \ref{eq:front_folding} to update $I$ and the following equation for $O$:
\begin{equation}
    \begin{cases}
    W \leftarrow \frac{W}{\gamma}(\sigma + \epsilon) \\
    b \leftarrow \frac{b}{\gamma}(\sigma + \epsilon) + \frac{\mu}{\sigma + \epsilon} \gamma - \beta
    \end{cases}
\end{equation}

Now we simply need to verify that any path in the graph associated to $f$ computes the same outputs before and after folding. As we only update $C$, it is enough to do the verification for any path in $C$. We limit the verification to the case study of Figure \ref{fig:proof} a).
Before BN folding, the graph corresponds to the mathematical function $f$. We note $W_i$ and $b_i$ the weights and biases of each expressive node of the graph. We note $\sigma$ the operation performed by the non-affine node and $g_1$ and $g_2$ the operations from the other nodes (green nodes). As such, we have
\begin{equation}
\begin{cases}
    G_1:x\mapsto g_1(W_2 (W_1 x + b_1) + b_2, W_3 (W_1 x + b_1) + b_3)\\
    G_2 :x \mapsto g_2(G_1(x))\\
    f:x \mapsto \sigma(BN(G_2(x))), \sigma(W_4G_2(x) + b_4)
\end{cases}
\end{equation}
where $G_1$ and $G_2$ correspond to the transformation of the inputs by the first green node and second green node, respectively. 
In this example, the BN layer satisfies the conditions of theorem \ref{thm:folding}.
We fold the BN layer in the layers of set $I$, \textit{i.e.} the second and third expressive layers, using equation \ref{eq:back_folding}. We note their weights and biases $\tilde W_i$ and $\tilde b_i$.
Finally, we adapt the layer of $O$, \textit{i.e.} the fourth expressive layer, to its new processed inputs using equation \ref{eq:update_back} and note $\bar W_4$ and $\bar b_4$.
The folded network will compute:
\begin{equation}
\begin{cases}
    \tilde G_1:x\mapsto g_1(\tilde W_2 (W_1 x + b_1) + \tilde b_2, \tilde W_3 (W_1 x + b_1) + \tilde b_3)\\
    \tilde G_2 :x \mapsto g_2(\tilde G_1(x)) = BN(G_2(x))\\
    \tilde f:x \mapsto \sigma(\tilde G_2(x)), \sigma(\bar W_4\tilde G_2(x) + \bar b_4)\\
    \qquad = \sigma(BN(G_2(x))) , \sigma(W_4 BN^{-1}(BN(G_2(x))) + b_4)\\
    \qquad = f(x)
\end{cases}
\end{equation}
where the folded operator $\tilde G_2$ is equal to $BN(G_2)$ before folding. Because we negated the folding in $W_4$ and $b_4$ the predictive function remains unchanged.

Consequently, the algorithm based on the proposed necessary and sufficient condition from theorem \ref{thm:folding} is optimal, \textit{i.e.} a BN layer can be folded if and only if 
\begin{enumerate}
    \item the set is not limited to $\{N\}$
    \item all the leaves (aside $N$) of the set are expressive layers 
\end{enumerate}
We now validate the empirical, cost-free, benefits resulting from BaN-OFF.
\section{Experiments}
\subsection{Toy}
We implemented the toy architectures described in Figure \ref{fig:bn_folding}. Table \ref{tab:toy} lists our results on an Intel CPU m3. The inference speed-up is measured as the following ratio $\frac{T_{\text{old}}-T_{\text{new}}}{T_{\text{old}}} \in [0;100\%]$ where $T_{\text{new}}$ is the inference time of the network after BN folding and $T_{\text{old}}$ is the original inference time, \textit{i.e.} the higher the better.

As expected the two algorithms, naive and BaN-OFF, share the same performance on the two first architectures as they perform the exact same transformations of the network. On the other hand, for the last configuration only the proposed method performed BN folding and resulted in inference speed increase.
We also observe that the reduction of inference time is not correlated with the proportion of removed parameters as BN layers don't require many parameters but may put sequential computations of the architecture on hold during the forward pass.
Furthermore, the predictions were only changed by up to $10^{-6}$ in $L^1$ norm on vector outputs of $1000$ elements.
This is due to numerical approximations.
Overall, this shows the importance of BN folding in neural networks acceleration as it doesn't change the predictive function and provide a significant boost in terms of speed.
\begin{table}[!t]
\caption{Comparison between the naive and proposed BN-folding algorithm in terms of increased inference speed on an Intel CPU m3. We also report the proportion of removed parameters.}
\label{tab:toy}
\centering
\setlength\tabcolsep{4pt}
\renewcommand{\arraystretch}{1.15}
\begin{tabular}{c|c|c|c}
\hline
Model & Naive & BaN-OFF & \% removed parameters\\
\hline
\hline
Figure \ref{fig:bn_folding} a) & \textbf{63.42} & \textbf{63.42} & 0.488 \\
Figure \ref{fig:bn_folding} b) & \textbf{51.44} & \textbf{51.44} & 0.488 \\
Figure \ref{fig:bn_folding} c) & 0 & \textbf{47.81} & 0.314 \\
\hline
\end{tabular}
\end{table}

\subsection{Main Results}
\begin{table}[!t]
\caption{Comparison between the naive and proposed BN-folding algorithm, on ResNets trained on Cifar10, in terms of percentage of relative reduced inference time on an Intel CPU m3. We also report the percentage of removed parameters.}
\label{tab:cifar10}
\centering
\setlength\tabcolsep{4pt}
\renewcommand{\arraystretch}{1.15}
\begin{tabular}{c|c|c|c}
\hline
Model & Naive & BaN-OFF & \% removed parameters\\
\hline
\hline
ResNet 20 & 30.64 & \textbf{35.04} & 0.61 \\
ResNet 56 & 43.19 & \textbf{50.97} & 0.63 \\
ResNet 110 & 49.21 & \textbf{54.36} & 0.63 \\
ResNet 164 & 34.35 & \textbf{39.20} & 0.63 \\
\hline
\end{tabular}
\end{table}

\paragraph{Cifar10:} Table \ref{tab:cifar10} summarizes our results on ResNets \cite{he2016deep} trained on Cifar10 \cite{krizhevsky2009learning}. For both algorithms, the folded networks share identical predictions with the original models.
We conducted our first experiments on a small CPU as it corresponds to the most challenging use-case.
On a small device, sequential operations, e.g. BN layers, are not bottleneck compared to large expressive layers.
Furthermore, using a small device induces instability on the results for both algorithms. Standard deviation values (over the percentage of relative reduced inference time per frame) range from $3.97$ to $13.65$. However the proposed method systematically outperforms the naive BN-folding algorithm.

Furthermore, we confirm the previous observation on the decorrelation between the proportion of removed parameters by the acceleration method and the improvement in inference speed. 
We also observe that the inference boost can't be naively deduced from simple architecture properties such as depth and wideness.
\begin{table}[!t]
\caption{Comparison between the naive and proposed BN-folding algorithm, on models trained on ImageNet, in terms of percentage of relative reduced inference time on an Intel CPU m3. We also report the percentage of removed parameters.}
\label{tab:imagenet}
\centering
\setlength\tabcolsep{4pt}
\renewcommand{\arraystretch}{1.15}
\begin{tabular}{c|c|c|c}
\hline
Model & Naive & BaN-OFF & \% removed params\\
\hline
\hline
MobileNet & 24.35 & \textbf{29.01} & 0.97\\
\hline
ResNet 50 & \textbf{6.48} & \textbf{6.48} & 0.61 \\
ResNet 101 & \textbf{10.34} & \textbf{10.34} & 0.63 \\
ResNet 152 & \textbf{21.31} & \textbf{21.31} & 0.63 \\
\hline
EfficientNet B0 & 17.10 & \textbf{18.27} & 0.79\\
EfficientNet B1 & 33.88 & \textbf{35.95} & 0.80\\
EfficientNet B5 & 19.93 & \textbf{21.38} & 0.57\\
\hline
\end{tabular}
\end{table}

\paragraph{ImageNet:} Table \ref{tab:imagenet} contains our results on ResNets \cite{he2016deep}, MobileNet v2 \cite{sandler2018mobilenetv2} and EfficientNets \cite{tan2019efficientnet} trained on ImageNet \cite{ImageNet_cvpr09}. Folded networks, with the naive and the proposed method, the resulting networks share identical predictions with the original model. Similarly to Cifar10, the results were obtained on an Intel m3 CPU.
Standard deviations values range between $1.14$ for the least accelerated models and $12.68$ for the most accelerated models. As expected, the proposed method systematically outperforms the naive approach on all runs.

In terms of proportion of removed parameters, the results are comparable to our results on Cifar10 which demonstrates that, even for more challenging tasks it matters more to remove a few meaningful parameters (e.g. BN layers) than a lot of parameters. This is based on the results obtained with pruning methods such as \cite{frankle2018lottery,lin2020hrank}.
We also note that the correlation between depth and acceleration due to BN-folding on ResNets may be a coincidence.

\subsection{Stability over Hardware}
\paragraph{Central Processing Unit (CPU) comparison:}
Our results, in Table \ref{tab:gpu}, show the stability of the inference boost across different CPUs.
We also note that the maximum standard deviation on the results drops from $13.65$ to $2.06$ on the larger CPU. These results show the difficulties to work on very small devices for edge computing and the importance of efficient acceleration processes.

\begin{table}[!t]
\caption{Comparison of relative runtime reduction between CPU and GPU devices, on different hardware and different architectures trained on Cifar10 and ImageNet.}
\label{tab:gpu}
\centering
\setlength\tabcolsep{4pt}
\renewcommand{\arraystretch}{1.15}
\begin{tabular}{c|c|c|c}
\hline
Model & CPU m3 & CPU i9-9900K & RTX 3090 \\
\hline
\hline
ResNet 20       & 35.04 & 38.85 & 22.03\\
ResNet 56       & 50.97 & 49.37 & 47.33\\
ResNet 110      & 54.36 & 51.20 & 40.86\\
ResNet 164      & 39.20 & 48.16 & 56.94\\
\hline
MobileNet       & 29.01 & 27.76 & 69.14\\
\hline
ResNet 50       & 6.48  & 18.54 & 72.83\\
ResNet 101      & 10.34 & 16.77 & 47.89\\
ResNet 152      & 21.31 & 17.82 & 40.97\\
\hline
EfficientNet B0 & 18.27 & 20.74 & 52.72\\
EfficientNet B1 & 35.95 & 20.51 & 43.92\\
EfficientNet B5 & 21.38 & 11.80 & 49.60\\
\hline
\end{tabular}
\end{table}

\paragraph{Graphical Processing Unit (GPU) comparison:} Table \ref{tab:gpu} presents the comparison between CPU and GPU devices response to BN folding.
In contrast with CPU devices, GPU are very sensitive to sequential computations and very efficient on parallel operations. This property translates as significantly larger inference speed-up on GPU when compared to CPU.
As expected from a larger device, the standard deviations on the results represent less than $5.04$\% of the average value which shows the stability of the inference speed-up across different runs.

Furthermore, such results motivate the use of acceleration frameworks comprising BN folding for efficient inference on cloud servers using large GPUs.

\subsection{Other Acceleration Methods}
BN usually constitutes a complementary step in pruning and quantization processes. However it is not systematically applied. We put the emphasis on the efficiency of BN-folding in reaching the goal of DNN acceleration by comparing the cost in terms of accuracy in order to reach similar acceleration with pruning and quantization methods.

\begin{figure}[!t]
    \centering
    \includegraphics[width = 0.9\linewidth]{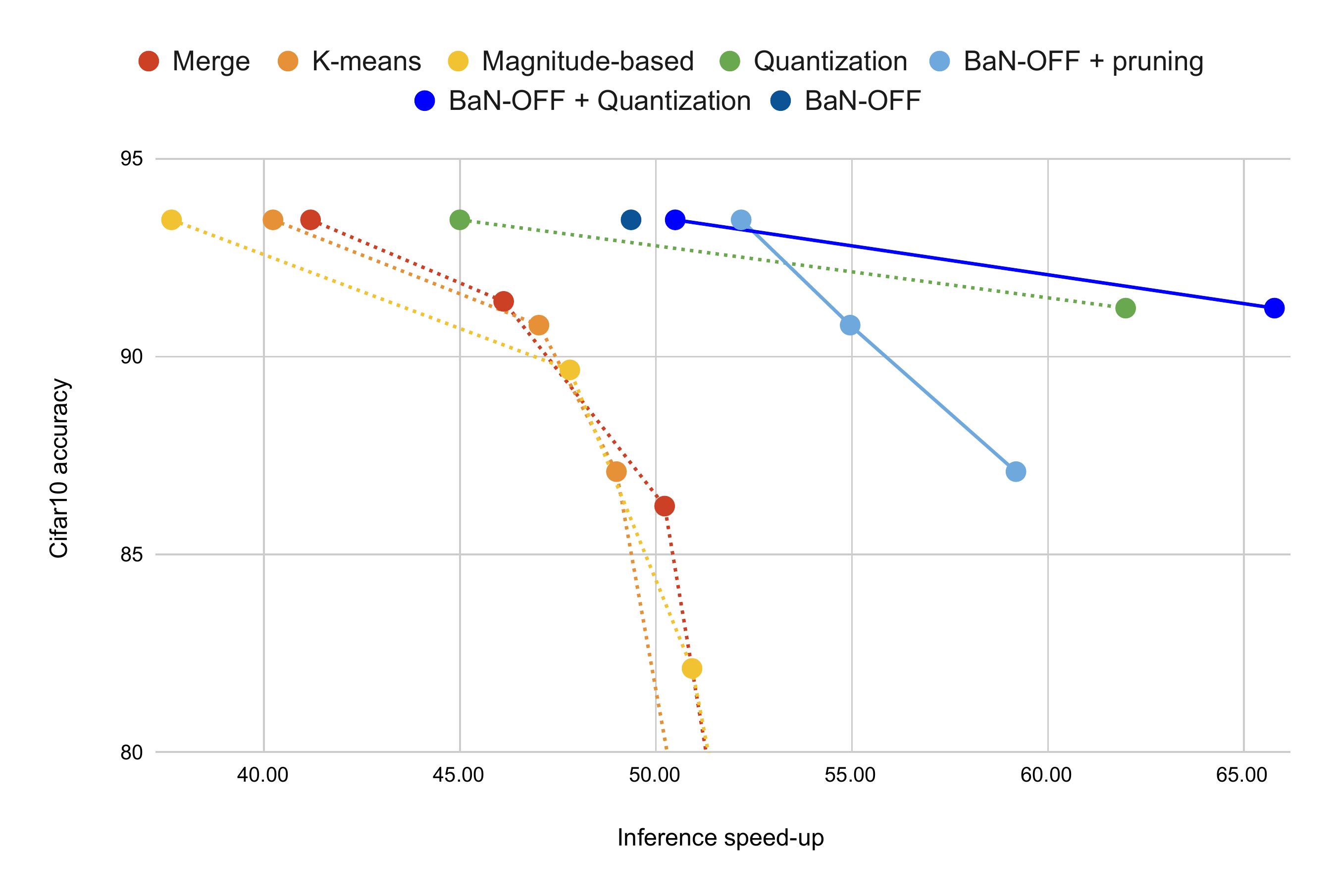}
     \caption{Comparison of different acceleration method on Cifar10 ResNet 56. Pruning and Quantization alone come at a cost in accuracy and need complex methodologies. Meanwhile, BaN-OFF comes at no cost and offer great improvements in runtime at no cost.}
    \label{fig:comparison_pruning_quantization}
\end{figure}
\paragraph{Pruning:} Figure \ref{fig:comparison_pruning_quantization} shows the comparison of the influence of pruning (merge step from \cite{yvinec2021red}) alone and BN folding on a ResNet 56 trained on Cifar10. For the sake of comparison, we implemented three classical pruning methods: Merge (which consists in merging similar neurons \cite{yvinec2021red}), K-means (which consists in clustering neurons by their weight values) and magnitude-based unstructured pruning (which consists in removing weight values based on their absolute value). These methods allows inference speed-ups up to $50.23\%$ but often significantly degrade the model accuracy. In comparison, the proposed BAN-OFF method allows similar runtime reduction without altering the predictive function. This suggests that, while the community focused on removing a lot of parameters whereas removing a few parameters under the form of BN layers folding can save a lot of computations by itself. Furthermore, BN folding and such pruning techniques can be combined for improved runtime reduction (up to $59.2 \%$).


\paragraph{Quantization:} We considered the naive quantization operator introduced in \cite{krishnamoorthi2018quantizing}. Such methods are usually more effective than pruning in terms of inference speed-up, which is confirmed in Figure \ref{fig:comparison_pruning_quantization}. Similarly to the previous results, BN folwing and quantization can be used in conjunction for improved runtime reduction ($50.5-65.8\%$).


\subsection{Other Data Applications}
\begin{table}[!t]
\caption{Results of the proposed method on object detection methods trained on MS COCO. We report the percentage of relative reduced inference time of both methods for Faster RCNN (MobileNet backbone), Mask RCNN (ResNet backbone), EfficientDet D0, D7 and CenterNet with a stacked hourglass.}
\label{tab:detec}
\centering
\setlength\tabcolsep{4pt}
\renewcommand{\arraystretch}{1.15}
\begin{tabular}{c|c|c|c}
\hline
Model & Naive & BaN-OFF & improvement \\
\hline
\hline
Faster rcnn (mob) & 51.78 & \textbf{52.40} & +0.63\\
Mask rcnn (res)   & \textbf{55.37} & \textbf{55.37} & +0.00\\
EfficientDet D0	  & 51.72 & \textbf{53.93} & +2.21\\
EfficientDet D7	  & 44.36 & \textbf{45.73} & +1.37\\
CenterNet         & 52.86 & \textbf{53.55} & +0.69\\
\hline
\end{tabular}
\end{table}
\paragraph{MS COCO:} Table \ref{tab:detec} summarizes our results on the detection task of MS COCO \cite{lin2014microsoft} on an RTX 3090. We considered widely used architectures such as Faster RCNN \cite{ren2015faster}, Mask RCNN \cite{he2017mask}, EfficientDet \cite{tan2020efficientdet} and CenterNet \cite{zhou2019objects}.
Similarly to ImageNet, the folded models and the original models share the same predictions and the same accuracy. We see results very comparable to the results from table \ref{tab:gpu} (column RTX 3090) which is understandable as the detection algorithm are very similar, in the sense of the presence of BN layers, to their corresponding backbones.

\begin{table}[!t]
\caption{Results of the proposed method on object detection methods trained on CityScapes. We report the percentage of relative reduced inference time of the two methods for U-Net with ResNet 34 and Efficient B0 backbones, FPN and PSPN with ResNet 34 backbone.}
\label{tab:seg}
\centering
\setlength\tabcolsep{4pt}
\renewcommand{\arraystretch}{1.15}
\begin{tabular}{c|c|c|c}
\hline
Model & run-T & BaN-OFF & improvement \\
\hline
\hline
U-Net (Res34) & 57.89 & \textbf{59.95} & +2.06\\
U-Net (effB0) & 59.77 & \textbf{62.86} & +3.09\\
FPN (Res34)   & 52.13 & \textbf{54.80} & +2.67\\
PSPN (Res34)  & 53.01 & \textbf{55.14} & +2.14\\
\hline
\end{tabular}
\end{table}

\paragraph{CityScapes:} From our results, listed in Table \ref{tab:seg}, on the semantic segmentation task of CityScapes \cite{cordts2016cityscapes}, we observe similar results to object detection. On both FPN \cite{lin2017feature} and PSPN \cite{zhao2017pyramid} models, the proposed method achieves remarkable results with 55\% reduction of the inference time on an RTX 3090.
Similar results are obtained on U-Net \cite{ronneberger2015u} with an EfficientNet B0 backbone with great stability across runs (standard deviation below 2 points).
Finally, on U-Net with ResNet 34 backbone, we reach our most remarkable result of almost 63\% reduction of latency with a standard deviation below 1 point.
\section{Conclusion}
In this work, we highlighted and addressed the limitations of existing batch normalization (BN) layer folding approaches. We argued that BN folding is often overlooked in the deep neural network acceleration literature, as it allows drastic runtime reduction at inference time, especially compared to other acceleration methods, e.g. pruning or quantization.

From a theoretical point of view, we proposed a novel necessary and sufficient condition for folding BN layers working for any architecture. We also derived a practical algorithm, dubbed BAN-OFF, for optimal BN folding. BAN-OFF allows to significantly reduce the inference time of any deep neural network regardless of the task at hand (shall it be image classification, object detection or semantic segmentation), the deep network architecture or the hardware of which it shall run on, at virtually no cost.

For reproducibility concerns, the source code will be available.

\bibliographystyle{named}
\bibliography{ijcai22}

\end{document}